\documentclass{eptcs}
\providecommand{\event}{TARK 2017} 
\usepackage{underscore}           

\usepackage{booktabs} 
\usepackage{amssymb}
\usepackage{amsmath}
\usepackage{wasysym}
\usepackage{epigraph}
\usepackage{graphics}
\usepackage{setspace}
\usepackage{fancybox} 
\newtheorem{definition}{Definition}

\newtheorem{lemma}{Lemma}

\newtheorem{corollary}{Corollary}
\newenvironment{proof}{\noindent{\sf Proof.}}{\hfill $\boxtimes\hspace{2mm}$\linebreak}
\renewcommand{\phi}{\varphi}
\renewcommand{\epsilon}{\varepsilon}

\newcommand{\K}{{\sf K}}
\renewcommand{\S}{{\sf S}}
\newcommand{\E}{{\sf H}}

\title{Together We Know How to Achieve:\\ An Epistemic Logic of Know-How\\
\vspace{3mm}
\Large (Extended Abstract)
}
\author{Pavel Naumov
\institute{Vassar College\\ Poughkeepsie, New York, USA}
\email{pnaumov@vassar.edu}
\and
 Jia Tao
\institute{Lafayette College\\
Easton, Pennsylvania, USA}
\email{\quad taoj@lafayette.edu}
}
\def\titlerunning{Together We Know How to Achieve}
\def\authorrunning{Pavel Naumov \& Jia Tao}

\begin{document}

\maketitle

\begin{abstract}
The existence of a coalition strategy to achieve a goal does not necessarily mean that the coalition has enough information to know how to follow the strategy. Neither does it mean that the coalition knows that such a strategy exists. The paper studies an interplay between the distributed knowledge, coalition strategies, and coalition ``know-how" strategies. The main technical result is a sound and complete trimodal logical system that describes the properties of this interplay.
\end{abstract}

\maketitle

\section{Introduction}


An agent $a$ comes to a fork in a road. There is a sign that says that one of the two roads leads to prosperity, another to death. The agent must take the fork, but she does not know which road leads where. Does the agent have a strategy to get to prosperity? On one hand, since one of the roads leads to prosperity, such a strategy clearly exists. We denote this fact by modal formula $\S_a p$, where statement $p$ is a claim of future prosperity. Furthermore, agent $a$ knows that such a strategy exists. We write this as $\K_a\S_a p$. Yet, the agent does not know what the strategy is and, thus, does not know how to use the strategy. We denote this by $\neg\E_a p$, where {\em know-how} modality $\E_a$ expresses the fact that agent $a$ knows how to achieve the goal based on the information available to her.  In this paper we study the interplay between modality $\K$, representing {\em knowledge}, modality $\S$, representing the existence of a {\em strategy}, and modality $\E$, representing the existence of a {\em know-how strategy}. Our main result is a complete trimodal axiomatic system capturing properties of this interplay.

\subsection{Epistemic Transition Systems}

In this paper we use epistemic transition systems to capture knowledge and strategic behavior. Informally, epistemic transition system is a directed labeled graph supplemented by an indistinguishability relation on vertices. For instance, our motivational example above can be captured by epistemic transition system $T_1$ depicted in Figure~\ref{intro-1 figure}. 
\begin{figure}[ht]
\begin{center}
\vspace{-2mm}
\scalebox{.6}{\includegraphics{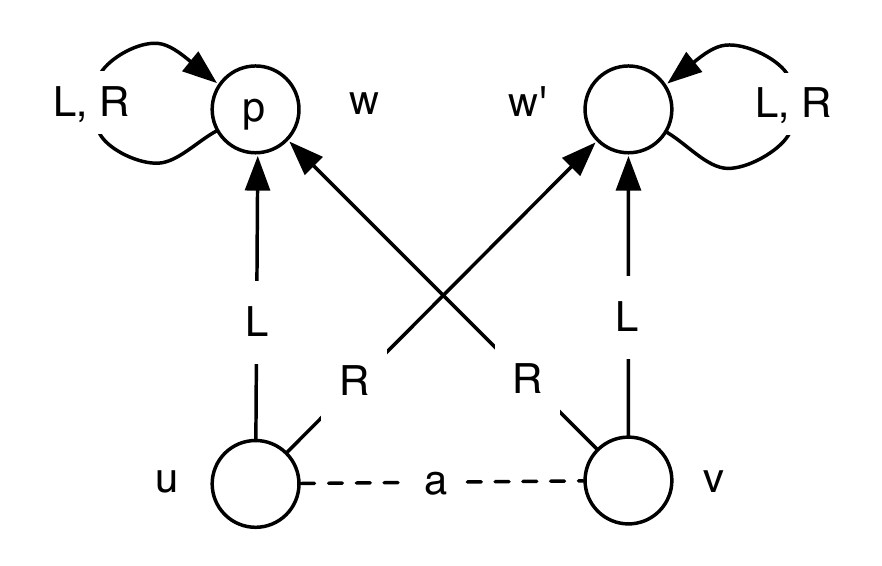}}
\vspace{0mm}
\caption{Epistemic transition system $T_1$.}\label{intro-1 figure}
\vspace{-2mm}
\end{center}
\vspace{-2mm}
\end{figure}
In this system state $w$ represents the prosperity and state $w'$ represents death. The original state is $u$, but it is indistinguishable by the agent $a$ from state $v$. Arrows on the diagram represent possible transitions between the states. Labels on the arrows represent the choices that the agents make during the transition. For example, if in state $u$ agent chooses left (L) road, she will transition to the prosperity state $w$ and if she chooses right (R) road, she will transition to the death state $w'$. In another epistemic state $v$, these roads lead the other way around. States $u$ and $v$ are not distinguishable by agent $a$, which is shown by the dashed line between these two states. In state $u$ as well as state $v$ the agent has a strategy to transition to the state of prosperity: $u\Vdash\S_a p$ and $v\Vdash\S_a p$. In the case of state $u$ this strategy is L, in the case of state $v$ the strategy is R. Since the agent cannot distinguish states $u$ and $v$, in both of these states she does not have a know-how strategy to reach prosperity: $u\nVdash\E_a p$ and $v\nVdash\E_a p$. At the same time, since formula $\S_a p$ is satisfied in all states  indistinguishable to agent $a$ from state $u$, we can claim that $u\Vdash\K_a\S_a p$ and, similarly, $v\Vdash\K_a\S_a p$. 

\begin{figure}[ht]
\begin{center}
\vspace{-2mm}
\scalebox{.6}{\includegraphics{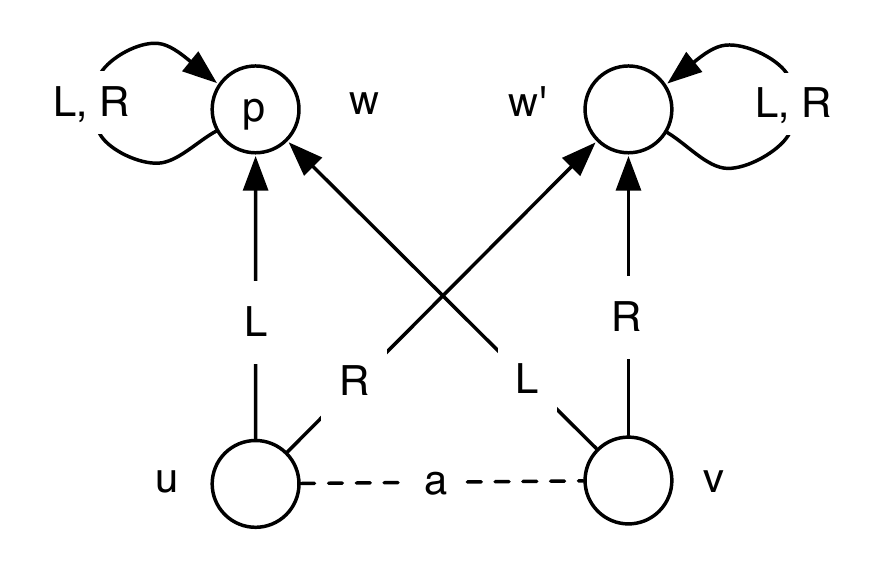}}
\vspace{0mm}
\caption{Epistemic transition system $T_2$.}\label{intro-2 figure}
\vspace{-2mm}
\end{center}
\vspace{-2mm}
\end{figure}

As our second example, let us consider the epistemic transition system $T_2$ obtained from $T_1$ by swapping labels on transitions from $v$ to $w$ and from $v$ to $w'$, see Figure~\ref{intro-2 figure}. Although in system $T_2$ agent $a$ still cannot distinguish states $u$ and $v$, she has a know-how strategy from either of these states to reach state $w$. We write this as $u\Vdash\E_a p$ and $v\Vdash\E_a p$. The strategy is to choose L. This strategy is know-how because it does not require to make different choices in the states that the agent cannot distinguish. 

\subsection{Imperfect Recall}
For the next example, we consider a transition system $T_3$ obtained from system $T_1$ by adding a new epistemic state $s$. From state $s$, agent $a$ can choose label L to reach state $u$ or choose label R to reach state $v$. Since proposition $q$ is satisfied in state $u$, agent $a$ has a know-how strategy to transition from state $s$ to a state (namely, state $u$) where $q$ is satisfied. Therefore, $s\Vdash\E_a q$. 

\begin{figure}[ht]
\begin{center}
\vspace{-2mm}
\scalebox{.6}{\includegraphics{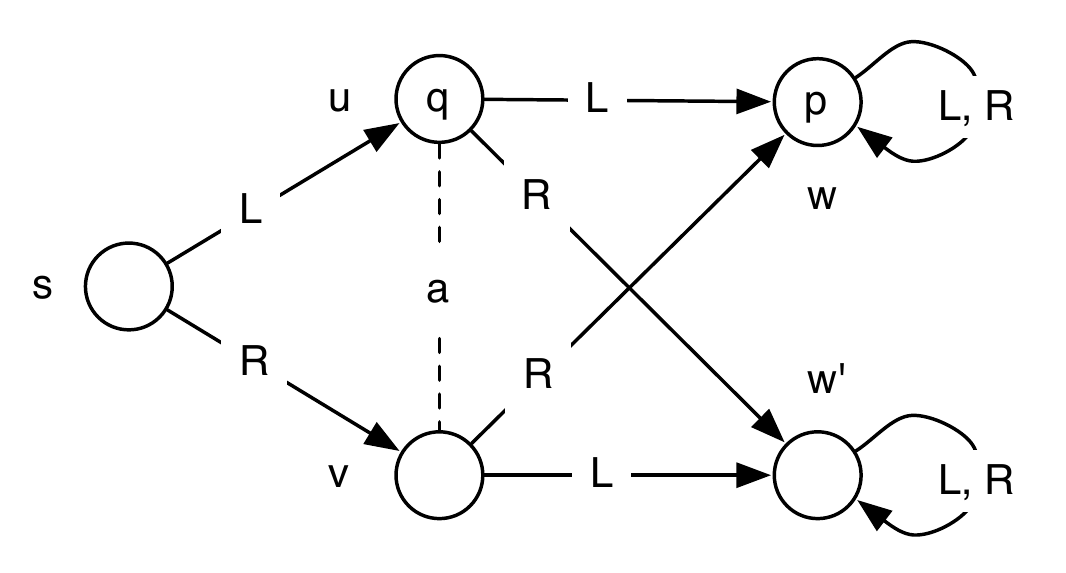}}
\vspace{0mm}
\caption{Epistemic transition system $T_3$.}\label{intro-3 figure}
\vspace{-2mm}
\end{center}
\vspace{-2mm}
\end{figure}

A more interesting question is whether $s\Vdash\E_a\E_a p$ is true. In other words, does agent $a$ know how to transition from state $s$ to a state in which she knows how to transition to another state in which $p$ is satisfied? One might think that such a strategy indeed exists: in state $s$ agent $a$ chooses  label L to transition to state $u$. Since there is no transition labeled by L that leads from state $s$ to state $v$, upon ending the first transition the agent would know that she is in state $u$, where she needs to choose label L to transition to state $w$. This argument, however, is based on the assumption that agent $a$ has a perfect recall. Namely, agent $a$ in state $u$ remembers the choice that she made in the previous state. We assume that the agents do not have a perfect recall and that an epistemic state description captures whatever memories the agent has in this state. In other words, in this paper we assume that the only knowledge that an agent possesses is the knowledge captured by the indistinguishability relation on the epistemic states. Given this assumption, upon reaching the state $u$ (indistinguishable from state $v$) agent $a$ knows that there {\em exists} a choice that she can make to transition to state in which $p$ is satisfied: $s\Vdash\E_a\S_a p$. However, she does not know which choice (L or R) it is: $s\nVdash\E_a\E_a p$. 

\subsection{Multiagent Setting}

\begin{figure}[ht]
\begin{center}
\vspace{-2mm}
\scalebox{.6}{\includegraphics{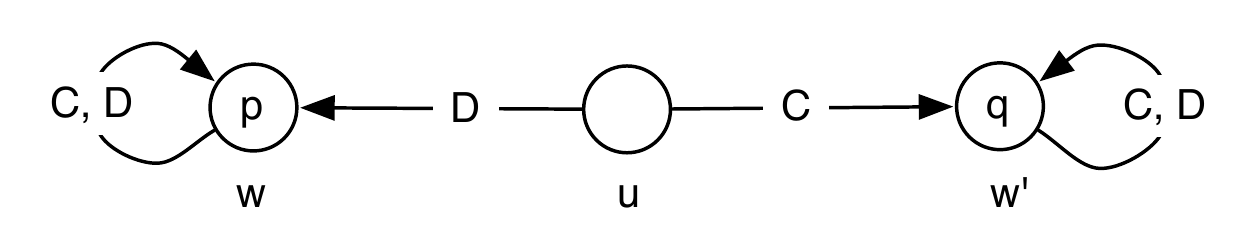}}
\vspace{0mm}
\caption{Epistemic transition system $T_4$.}\label{intro-4 figure}
\vspace{-2mm}
\end{center}
\vspace{-2mm}
\end{figure}

So far, we have assumed that only agent $a$ has an influence on which transition the system takes. In transition system $T_4$ depicted in Figure~\ref{intro-4 figure}, we introduce another agent $b$ and assume both agents $a$ and $b$ have influence on the transitions. In each state, the system takes the transition labeled D by default unless there is a consensus of agents $a$ and $b$ to take the transition labeled C. In such a setting, each agent has a strategy to transition system from state $u$ into state $w$ by voting D, but neither of them alone has a strategy to transition from state $u$ to state $w'$ because such a transition requires the consensus of both agents. Thus, $u\Vdash\S_a p\wedge \S_b p\wedge \neg\S_a q\wedge \neg\S_b q$. Additionally, both agents know how to transition the system from state $u$ into state $w$, they just need to vote D. Therefore, $u\Vdash\E_a p\wedge \E_b p$.

\begin{figure}[ht]
\begin{center}
\vspace{-2mm}
\scalebox{.6}{\includegraphics{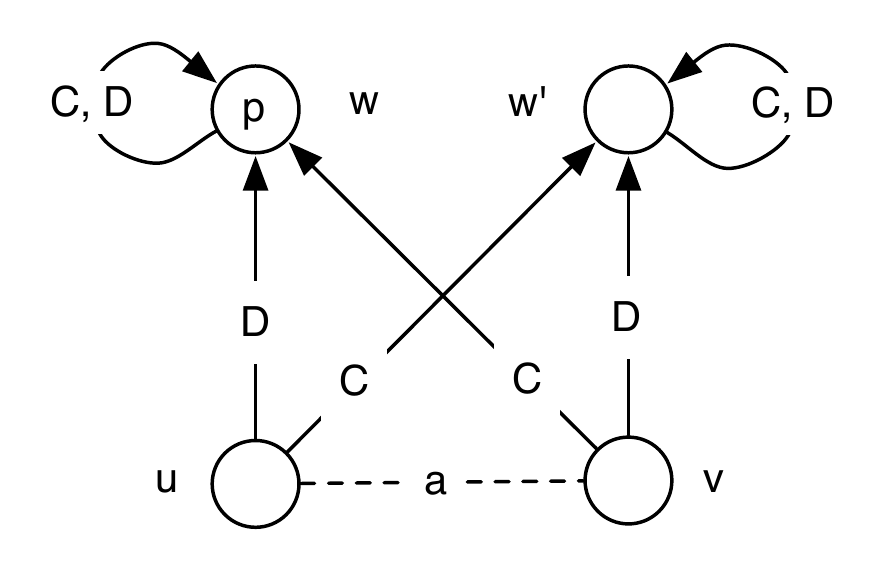}}
\vspace{0mm}
\caption{Epistemic transition system $T_5$.}\label{intro-5 figure}
\vspace{-2mm}
\end{center}
\vspace{-2mm}
\end{figure}

In Figure~\ref{intro-5 figure}, we show a more complicated transition system obtained from $T_1$ by renaming label L to D and renaming label R to C. Same as in transition system $T_4$, we assume that there are two agents $a$ and $b$ voting on the system transition. We also assume that agent $a$ cannot distinguish states $u$ and $v$ while agent $b$ can. By default, the system takes the transition labeled D unless there is a consensus to take transition labeled C. As a result, agent $a$ has a strategy (namely, vote D) in state $u$ to transition system to state $w$, but because agent $a$ cannot distinguish state $u$ from state $v$, not only does she not know how to do this, but she is not aware that such a strategy exists: $u\Vdash\S_a p\wedge\neg\E_a p \wedge\neg\K_a\S_a p$. Agent $b$, however, not only has a strategy to transition the system from state $u$ to state $w$, but also knows how to achieve this: $u\Vdash \E_b p$.

\subsection{Coalitions}

We have talked about strategies, know-hows, and knowledge of individual agents. In this paper we consider knowledge, strategies, and know-how strategies of coalitions. There are several forms of group knowledge that have been studied before. The two most popular of them are common knowledge and distributed knowledge~\cite{fhmv95}. Different contexts call for different forms of group knowledge.

As illustrated in the famous Two Generals' Problem~\cite{aeh75sigop,g78os} where communication channels between the agents are unreliable, establishing a common knowledge between agents might be essential for having a strategy. 

In some settings, the distinction between common and distributed knowledge is insignificant. For example, if members of a political fraction get together to share {\em all} their information and to develop a common strategy, then the distributed knowledge of the members becomes the common knowledge of the fraction during the in-person meeting.

Finally, in some other situations the distributed knowledge makes more sense than the common knowledge. For example, if a panel of experts is formed to develop a strategy, then this panel achieves the best result if it relies on the combined knowledge of its members rather than on their common knowledge.

In this paper we focus on distributed coalition knowledge and distributed-know-how strategies. We leave the common knowledge for the future research.

To illustrate how distributed knowledge of coalitions interacts with strategies and know-hows, consider epistemic transition system $T_6$ depicted in Figure~\ref{intro-6 figure}. In this system, agents $a$ and $b$ cannot distinguish states $u$ and $v$ while agents $b$ and $c$ cannot distinguish states $v$ and $u'$. In every state, each of agents $a$, $b$ and $c$ votes either L or R, and the system transitions according to the majority vote. In such a setting, any coalition of two agents can fully control the transitions of the system. 

\begin{figure}[ht]
\begin{center}
\vspace{-2mm}
\scalebox{.6}{\includegraphics{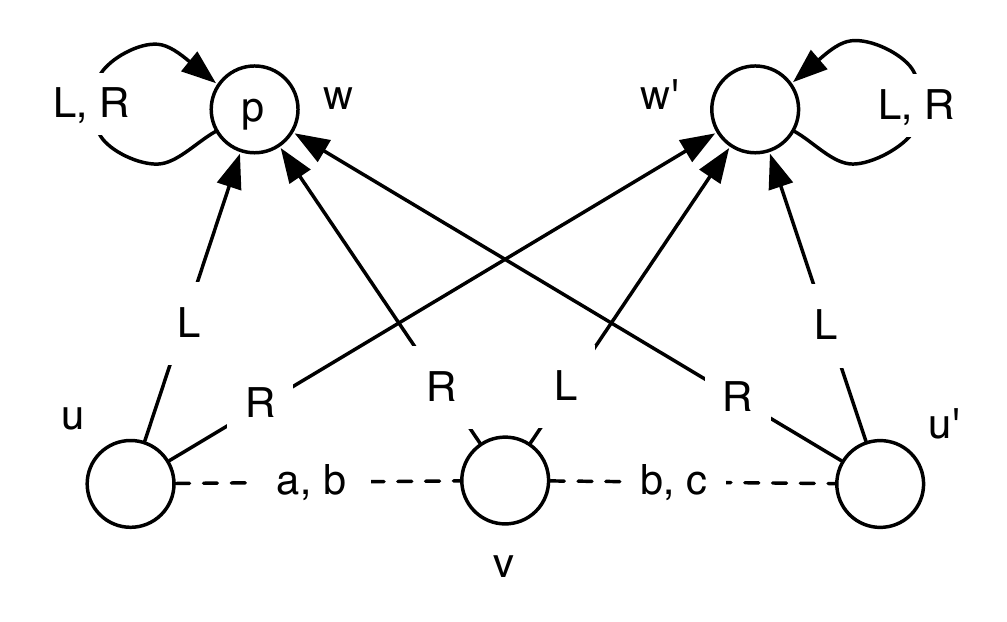}}
\vspace{0mm}
\caption{Epistemic transition system $T_6$.}\label{intro-6 figure}
\vspace{-2mm}
\end{center}
\vspace{-2mm}
\end{figure}

For example, by both voting L, agents $a$ and $b$ form a coalition $\{a,b\}$ that forces the system to transition from state $u$ to state $w$ no matter how agent $c$ votes. Since proposition $p$ is satisfied in state $w$, we write $u\Vdash\S_{\{a,b\}} p$, or simply $u\Vdash\S_{a,b} p$. Similarly, coalition $\{a,b\}$ can vote R to force the system to transition from state $v$ to state $w$. Therefore, coalition $\{a,b\}$ has strategies to achieve $p$ in states $u$ and $v$, but the strategies are different. Since they cannot distinguish states $u$ and $v$, agents $a$ and $b$ know that they have a strategy to achieve $p$, but they do \emph{not} know how to achieve $p$. In our notations, $v\Vdash S_{a,b}p\wedge \K_{a,b}S_{a,b}p \wedge \neg\E_{a,b} p$.   

On the other hand, although agents $b$ and $c$ cannot distinguish states $v$ and $u'$, by both voting R in either of states $v$ and $u'$, they form a coalition $\{b, c\}$ that forces the system to transition to state $w$ where $p$ is satisfied. Therefore, in any of states $v$ and $u'$, they not only have a strategy to achieve $p$, but also know that they have such a strategy, and more importantly, they know how to achieve $p$, that is, $v\Vdash\E_{b,c} p$.

\subsection{Nondeterministic Transitions}

In all the examples that we have discussed so far, given any state in a system, agents' votes uniquely determine the transition of the system. Our framework also allows nondeterministic transitions. Consider transition system $T_7$ depicted in Figure~\ref{intro-7 figure}. In this system, there are two agents $a$ and $b$ who can vote either C or D. If both agents vote C, then the system takes one of the consensus transitions labeled with C. Otherwise, the system takes the transition labeled with D. Note that there are two consensus transitions starting from state $u$. Therefore, even if both agents vote C, they do not have a strategy to achieve $p$, i.e., $u\nVdash\S_{a,b}p$. However, they can achieve $p\vee q$. Moreover, since all agents can distinguish all states, we have $u \Vdash\E_{a,b}(p\vee q)$.

\begin{figure}[ht]
\begin{center}
\vspace{-2mm}
\scalebox{.6}{\includegraphics{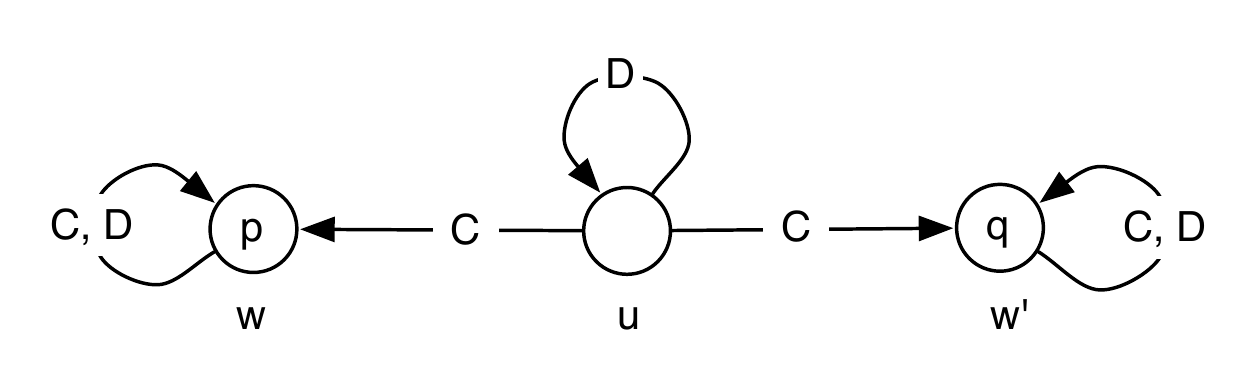}}
\vspace{0mm}
\caption{Epistemic transition system $T_7$.}\label{intro-7 figure}
\vspace{-2mm}
\end{center}
\vspace{-2mm}
\end{figure}

\subsection{Universal Principles}

In the examples above we focused on specific properties that were either satisfied or not satisfied in particular states of epistemic transition systems $T_1$ through $T_7$.
In this paper, we study properties that are satisfied in all states of all epistemic transition systems. Our main result is a sound and complete axiomatization of all such properties. We finish the introduction with an informal discussion of these properties.

\paragraph{Properties of Single Modalities}
Knowledge modality $K_C$ satisfies the axioms of epistemic logic S5 with distributed knowledge. Both strategic modality $S_C$ and know-how modality $\E_C$ satisfy cooperation properties~\cite{p01illc,p02}:  
\begin{eqnarray}
&\S_C(\phi\to\psi)\to(\S_D\phi\to\S_{C\cup D}\psi), \mbox{ where } C\cap D=\varnothing,\label{s coop}\\
&\E_C(\phi\to\psi)\to(\E_D\phi\to\E_{C\cup D}\psi), \mbox{ where } C\cap D=\varnothing.\label{e coop}
\end{eqnarray}
They also satisfy monotonicity properties
\begin{eqnarray*}
\S_C\phi\to\S_D\phi, \mbox{ where } C\subseteq D,\\
\E_C\phi\to\E_D\phi, \mbox{ where } C\subseteq D.
\end{eqnarray*}
The two monotonicity properties are not among the axioms of our logical system because, as we show in Lemma~\ref{subset lemma S} and Lemma~\ref{subset lemma E}, they are derivable.

\paragraph{Properties of Interplay}

Note that $w\Vdash\E_C\phi$ means that coalition $C$ has the same strategy to achieve $\phi$ in all epistemic states indistinguishable by the coalition from state $w$. Hence, the following principle is universally true:
\begin{equation}\label{st pos intro}
    \E_C\phi\to K_C\E_C\phi.
\end{equation}
Similarly, $w\Vdash\neg\E_C\phi$ means that coalition $C$ does not have the same strategy to achieve $\phi$ in all epistemic states indistinguishable by the coalition from state $w$. Thus,
\begin{equation}\label{st neg intro}
    \neg\E_C\phi\to K_C\neg\E_C\phi.
\end{equation}
We call properties~(\ref{st pos intro}) and (\ref{st neg intro}) {\em strategic positive introspection} and {\em strategic negative introspection}, respectively. 
The strategic negative introspection is one of our axioms. Just as how the positive introspection principle follows from the rest of the axioms in S5, the strategic positive introspection principle is also derivable (see Lemma~\ref{strategic positive introspection lemma}).

Whenever a coalition knows how to achieve something, there should exist a strategy for the coalition to achieve. In our notation,
\begin{equation}\label{st truth}
    \E_C\phi\to\S_C\phi.
\end{equation}
We call this formula {\em strategic truth} property and it is one of the axioms of our logical system.

The last two axioms of our logical system deal with empty coalitions. First of all, if formula $\K_\varnothing\phi$ is satisfied in an epistemic state of our transition system, then formula $\phi$ must be satisfied in every state of this system. Thus, even empty coalition has a trivial strategy to achieve $\phi$:
\begin{equation}\label{empty coal}
    \K_\varnothing\phi\to\E_\varnothing\phi. 
\end{equation}
We call this property {\em empty coalition} principle. In this paper we assume that an epistemic transition system never halts. That is, in every state of the system no matter what the outcome of the vote is, there is always a next state for this vote. This restriction on the transition systems yields property 
\begin{equation}\label{nonerm}
    \neg\S_C\bot. 
\end{equation}
that we call {\em nontermination} principle.

Let us now turn to the most interesting and perhaps most unexpected property of interplay. Note that $\S_\varnothing\phi$ means that an empty coalition has a strategy to achieve $\phi$. Since the empty coalition has no members, nobody has to vote in a particular way. Statement $\phi$ is guaranteed to happen anyway. Thus, statement $\S_\varnothing\phi$ simply means that statement $\phi$ is unavoidably satisfied after any single transition. 

\begin{figure}[ht]
\begin{center}
\vspace{-2mm}
\scalebox{.6}{\includegraphics{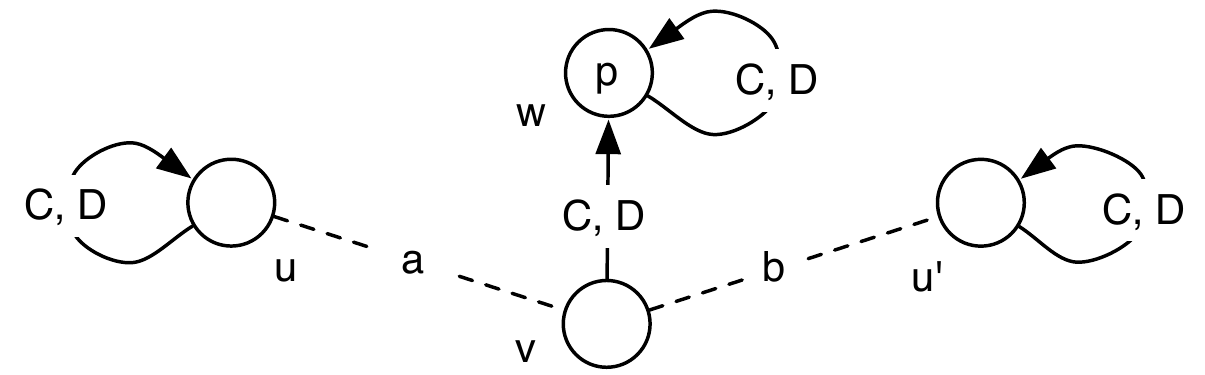}}
\vspace{0mm}
\caption{Epistemic transition system $T_8$.}\label{intro-8 figure}
\vspace{-2mm}
\end{center}
\vspace{-2mm}
\end{figure}
For example, consider an epistemic transition system depicted in Figure~\ref{intro-8 figure}. As in some of our earlier examples, this system has agents $a$ and $b$ who vote either C or D. If both agents vote C, then the system takes one of the consensus transitions labeled with C. Otherwise, the system takes the default transition labeled with D. Note that in state $v$ it is guaranteed that statement $p$ will happen after a single transition. Thus, $v\Vdash\S_\varnothing p$. At the same time, neither agent $a$ nor agent $b$ knows about this because they cannot distinguish state $v$ from states $u$ and $u'$ respectively. Thus, $v\Vdash\neg\K_a\S_\varnothing p \wedge \neg\K_b\S_\varnothing p$. 

In the same transition system $T_8$, agents $a$ and $b$ together can distinguish state $v$ from states $u$ and $u'$. Thus, $v\Vdash\K_{a,b}\S_\varnothing p$. 
In general,  statement $\K_C\S_\varnothing\phi$ means that not only $\phi$ is unavoidable, but coalition $C$ knows about it. Thus, coalition $C$ has a know-how strategy to achieve $\phi$:
$$
\K_C\S_\varnothing\phi\to \E_C\phi.
$$
In fact, the coalition would achieve the result no matter which strategy it uses. Coalition $C$ can even use a strategy that simultaneously achieves another result in addition to $\phi$: 
\begin{equation*}
    \K_C\S_\varnothing \phi\wedge \E_C\psi \to\E_C(\phi \wedge \psi).
\end{equation*}
In our logical system we use an equivalent form of the above principle that is stated using only implication:
\begin{equation}\label{ep determ}
    \E_C(\phi\to\psi)\to(\K_C\S_\varnothing \phi\to\E_C\psi).
\end{equation}
We call this property {\em epistemic determinicity} principle. Properties~(\ref{s coop}), (\ref{e coop}), (\ref{st neg intro}), (\ref{st truth}), (\ref{empty coal}), (\ref{nonerm}), and (\ref{ep determ}), together with axioms of epistemic logic S5 with distributed knowledge and propositional tautologies constitute the axioms of our sound and complete logical system. 

\subsection{Literature Review}

Logics of coalition power were developed by Marc Pauly~\cite{p01illc,p02}, who also proved the completeness of the basic logic of coalition power. 
Pauly's approach has been widely studied in the literature~\cite{g01tark,vw05ai,b07ijcai,sgvw06aamas,abvs10jal,avw09ai,b14sr}. An alternative logical system  was proposed by More and Naumov~\cite{mn12tocl}. 

Alur, Henzinger, and Kupferman introduced Alternating-Time Temporal Logic (ATL) that combines temporal and coalition modalities~\cite{ahk02}.
Van der Hoek and Wooldridge proposed to combine ATL with epistemic modality to form Alternating-Time Temporal Epistemic Logic~\cite{vw03sl}. They did not prove the completeness theorem for the proposed logical system. 

{\AA}gotnes and Alechina proposed a complete logical system that combines the coalition power and epistemic modalities~\cite{aa12aamas}. Since this system does not have epistemic requirements on strategies, it does not contain any axioms describing the interplay of these modalities. 

Know-how strategies were studied before under different names. While Jamroga and {\AA}gotnes talked about ``knowledge to identify and execute a strategy"~\cite{ja07jancl},  Jamroga and van der Hoek discussed ``difference between an agent knowing that he has a suitable strategy and knowing the strategy itself"~\cite{jv04fm}. Van Benthem called such strategies ``uniform"~\cite{v01ber}. Wang gave a complete axiomatization of ``knowing how" as a binary modality~\cite{w15lori,w17synthese}, but his logical system does not include the knowledge modality.  

In our AAMAS'17 paper, we investigated coalition strategies to enforce a condition indefinitely~\cite{nt17aamas}. Such strategies are similar to ``goal maintenance" strategies in Pauly's ``extended coalition logic"~\cite[p. 80]{p01illc}. We focused on ``executable" and ``verifiable" strategies. Using the language of the current paper, executability means that a coalition remains ``in the know-how" throughout the execution of the strategy. Verifiability means that the coalition can verify that the enforced condition remains true. In the notations of the current paper, the existence of a verifiable strategy could be expressed as $\S_C\K_C\phi$.  In~\cite{nt17aamas}, we provided a complete logical system that describes the interplay between the modality representing the existence of an ``executable" and ``verifiable" coalition strategy to enforce and the modality representing knowledge. This system can prove principles similar to the strategic positive introspection~(\ref{st pos intro}) and the strategic negative introspection~(\ref{st neg intro}) mentioned above.  

In the current paper, we combine know-how modality $\E$ with strategic modality $\S$ and epistemic modality $\K$. The proof of the completeness theorem is significantly more challenging than the one in \cite{nt17aamas}. It employs new techniques that construct pairs of maximal consistent sets in ``harmony" and in ``complete harmony", which are discussed in the full version of this paper~\cite{nt17arxiv-together}. 

\subsection{Paper Outline}

This paper is organized as follows. In Section~\ref{syntax and semantics} we introduce formal syntax and semantics of our logical system. In Section~\ref{axioms section} we list axioms and inference rules of the system. Section~\ref{examples section} provides examples of formal proofs in our logical systems. Section~\ref{conclusion section} concludes the paper.

The proofs of the soundness and the completeness can be found in the full version of this paper~\cite{nt17arxiv-together}. The key part of the proof of the completeness is the construction of a pair of sets in complete harmony.

\section{Syntax and Semantics}\label{syntax and semantics}

In this section we present the formal syntax and semantics of our logical system given a fixed finite set of agents $\mathcal{A}$. Epistemic transition system could be thought of as a Kripke model of modal logic S5 with distributed knowledge to which we add transitions controlled by a vote aggregation mechanism. Examples of vote aggregation mechanisms that we have considered in the introduction are the consensus/default mechanism and the majority vote mechanism. Unlike the introductory examples, in the general definition below we assume that at different states the mechanism might use different rules for vote aggregation. The only restriction on the mechanism that we introduce is that there should be at least one possible transition that the system can take no matter what the votes are. In other words, we assume that the system can never halt.  

For any set of votes $V$, by $V^\mathcal{A}$ we mean the set of all functions from set $\mathcal{A}$ to set $V$. Alternatively, the set $V^\mathcal{A}$ could be thought of as a set of tuples of elements of $V$ indexed by elements of $\mathcal{A}$.

\begin{definition}\label{transition system}
A tuple $(W,\{\sim_a\}_{a\in \mathcal{A}},V,M,\pi)$ is called an epistemic transition system, where
\begin{enumerate}
    \item $W$ is a set of epistemic states,
    \item $\sim_a$ is an indistinguishability equivalence relation on $W$ for each $a\in\mathcal{A}$,
    \item $V$ is a nonempty set called ``domain of choices", 
    \item $M\subseteq W\times V^\mathcal{A}\times W$ is an aggregation mechanism where for each $w\in W$ and each $\mathbf{s}\in V^\mathcal{A}$, there is $w'\in W$  such that $(w,\mathbf{s},w')\in M$,
    \item $\pi$ is a function that maps propositional variables into subsets of $W$.
\end{enumerate}
\end{definition}

\begin{definition}
A coalition is a subset of $\mathcal{A}$.
\end{definition}

Note that a coalition is always finite due to our assumption that the set of all agents $\mathcal{A}$ is finite. Informally, we say that two epistemic states are indistinguishable by a coalition $C$ if they are indistinguishable by every member of the coalition. Formally, coalition indistinguishability is defined as follows:

\begin{definition}\label{sim set}
For any epistemic states $w_1,w_2\in W$ and any coalition $C$, let $w_1\sim_C w_2$ if $w_1\sim_a w_2$ for each agent $a\in C$.
\end{definition}

\begin{corollary}\label{sim set corollary}
Relation $\sim_C$ is an equivalence relation on the set of states $W$ for each coalition $C$. 
\end{corollary}

By a strategy profile $\{s_a\}_{a\in C}$ of a coalition $C$ we mean a tuple that specifies vote $s_a\in V$ of each member $a\in C$. Since such a tuple can also be viewed as a function from set $C$ to set $V$, we denote the set of all strategy profiles of a coalition $C$ by $V^C$:

\begin{definition}\label{strategy}
Any tuple $\{s_a\}_{a\in C}\in V^C$ is called a strategy profile of coalition $C$.
\end{definition}

In addition to a fixed finite set of agents $\mathcal{A}$ we  also assume a fixed countable set of propositional variables.  The language $\Phi$ of our formal logical system is specified in the next definition.   

\begin{definition}\label{Phi}
Let $\Phi$ be the minimal set of formulae such that
\begin{enumerate}
    \item $p\in\Phi$ for each propositional variable $p$,
    \item $\neg\phi,\phi\to\psi\in\Phi$ for all formulae $\phi,\psi\in\Phi$,
    \item $\K_C\phi,\S_C\phi,\E_C\phi\in\Phi$ for each coalition $C$ and each $\phi\in\Phi$.
\end{enumerate}
\end{definition}

In other words, language $\Phi$ is defined by the following grammar:
$$
\phi := p\;|\;\neg\phi\;|\;\phi\to\phi\;|\;\K_C\phi\;|\;\S_C\phi\;|\;\E_C\phi.
$$

By $\bot$ we denote the negation of a tautology. For example, we can assume that $\bot$ is $\neg(p\to p)$ for some fixed propositional variable $p$. 

According to Definition~\ref{transition system}, a mechanism specifies the transition that a system might take for any strategy profile of the set of {\em all} agents $\mathcal{A}$. It is sometimes convenient to consider transitions that are {\em consistent} with a given strategy profile $\mathbf s$ of a give coalition $C\subseteq \mathcal{A}$. We write $w\to_{\mathbf s}u$ if a transition from state $w$ to state $u$ is consistent with strategy profile $\mathbf s$. The formal definition is below.

\begin{definition}\label{single arrow}
For any epistemic states $w,u\in W$, any coalition $C$, and any strategy profile ${\mathbf s}=\{s_a\}_{a\in C}\in V^C$, we write $w\to_{\mathbf s}u$ if $(w,\mathbf{s'},u)\in M$ for some strategy profile $\mathbf{s'}=\{s'_a\}_{a\in\mathcal{A}}\in V^\mathcal{A}$  such that $s'_a=s_a$ for each $a\in C$.
\end{definition}
\begin{corollary}\label{empty coalition corollary}
For any strategy profile $\mathbf{s}$ of the empty coalition $\varnothing$, if there are a coalition $C$ and a strategy profile $\mathbf{s'}$ of coalition $C$ such that $w\to_\mathbf{s'} u$, then $w\to_{\mathbf s}u$.
\end{corollary}

 The next definition is the key definition of this paper. It formally specifies the meaning of the three modalities in our logical system.



\begin{definition}\label{sat}
For any epistemic state $w\in W$ of a transition system $(W,\{\sim_a\}_{a\in \mathcal{A}},V,M,\pi)$ and any formula $\phi\in \Phi$, let relation $w\Vdash\phi$ be defined as follows
\begin{enumerate}
    \item $w\Vdash p$ if $w\in \pi(p)$ where $p$ is a propositional variable,
    \item $w\Vdash\neg\phi$ if $w\nVdash\phi$,
    \item $w\Vdash\phi\to\psi$ if $w\nVdash\phi$ or $w\Vdash\psi$,
    \item $w\Vdash \K_C\phi$ if $w'\Vdash\phi$ for each $w'\in W$ such that $w\sim_C w'$,
    \item $w\Vdash\S_C\phi$ if there is a strategy profile $\mathbf{s}\in V^C$ such that  $w\to_{\mathbf{s}} w'$ implies $w'\Vdash\phi$ for every $w'\in W$,
    \item $w\Vdash\E_C\phi$ if there is a strategy profile $\mathbf s\in V^C$ such that $w\sim_C w'$ and $w'\to_{\mathbf s} w''$ imply $w''\Vdash\phi$ for all $w',w''\in W$.
\end{enumerate}
\end{definition}

\section{Axioms}\label{axioms section}

In additional to propositional tautologies in language $\Phi$, our logical system consists of the following axioms. 
\begin{enumerate}
    \item Truth: $\K_C\phi\to\phi$,
    \item Negative Introspection: $\neg\K_C\phi\to\K_C\neg\K_C\phi$,
    \item Distributivity: $\K_C(\phi\to\psi)\to(\K_C\phi\to\K_{C}\psi)$,
    \item Monotonicity: $\K_C\phi\to \K_D\phi$, if $C\subseteq D$,
    \item Cooperation: $\S_C(\phi\to\psi)\to(\S_D\phi\to\S_{C\cup D}\psi)$, where $C\cap D=\varnothing$.
    \item Strategic Negative Introspection: $\neg\E_C\phi\to\K_C\neg\E_C\phi$,
    \item Epistemic Cooperation: $\E_C(\phi\to\psi)\to(\E_D\phi\to\E_{C\cup D}\psi)$, where $C\cap D=\varnothing$,
    \item Strategic Truth: $\E_C\phi\to\S_C\phi$,
    \item Epistemic Determinicity:\\ $\E_C(\phi\to\psi)\to(\K_C\S_\varnothing \phi\to\E_C\psi)$,
    \item Empty Coalition: $\K_\varnothing\phi\to\E_\varnothing \phi$,
    \item Nontermination: $\neg\S_C\bot$.
\end{enumerate}
We have discussed the informal meaning of these axioms in the introduction. In the full version of this paper~\cite{nt17arxiv-together}, we formally prove the soundness of these axioms with respect to the semantics from Definition~\ref{sat}.

We write $\vdash \phi$ if formula $\phi$ is provable from the axioms of our logical system using 
Necessitation, Strategic Necessitation, and Modus Ponens inference rules:
$$
\dfrac{\phi}{\K_C\phi}
\hspace{10mm}
\dfrac{\phi}{\E_C\phi}
\hspace{10mm}
\dfrac{\phi,\hspace{5mm} \phi\to\psi}{\psi}.
$$
We write $X\vdash\phi$ if formula $\phi$ is provable from the theorems of our logical system and a set of additional axioms $X$ using only Modus Ponens inference rule.

\section{Derivation Examples}\label{examples section}

In this section we give examples of formal derivations in our logical system. In Lemma~\ref{strategic positive introspection lemma} we prove the strategic positive introspection principle~(\ref{st pos intro}) discussed in the introduction. 

\begin{lemma}\label{strategic positive introspection lemma}
$\vdash \E_C\phi\to\K_C\E_C\phi$.
\end{lemma}
\begin{proof}
Note that formula $\neg\E_C\phi\to\K_C\neg\E_C\phi$ is an instance of Strategic Negative Introspection axiom. Thus, $\vdash \neg\K_C\neg\E_C\phi\to \E_C\phi$ by the law of contrapositive in the propositional logic. Hence,
$\vdash \K_C(\neg\K_C\neg\E_C\phi\to \E_C\phi)$ by  Necessitation inference rule. Thus, by  Distributivity axiom and Modus Ponens inference rule, 
\begin{equation}\label{pos intro eq}
   \vdash  \K_C\neg\K_C\neg\E_C\phi\to \K_C\E_C\phi.
\end{equation}

At the same time, $\K_C\neg\E_C\phi\to\neg\E_C\phi$ is an instance of Truth axiom. Thus, $\vdash \E_C\phi\to\neg\K_C\neg\E_C\phi$ by contraposition. Hence, taking into account the following instance of Negative Introspection axiom $\neg\K_C\neg\E_C\phi\to\K_C\neg\K_C\neg\E_C\phi$,
one can conclude that $\vdash \E_C\phi\to\K_C\neg\K_C\neg\E_C\phi$. The latter, together with statement~(\ref{pos intro eq}), implies the statement of the lemma by the laws of propositional reasoning.
\end{proof}

In the next example, we show that the existence of a know-how strategy by a coalition implies that the coalition has a distributed knowledge of the existence of a strategy.

\begin{lemma}
$\vdash \E_C\phi\to\K_C\S_C\phi$.
\end{lemma}
\begin{proof}
By Strategic Truth axiom, $\vdash \E_C\phi\to\S_C\phi$. Hence, $\vdash \K_C(\E_C\phi\to\S_C\phi)$ by Necessitation inference rule. Thus, $\vdash \K_C\E_C\phi\to\K_C\S_C\phi$ by Distributivity axiom and Modus Ponens inference rule. At the same time, $\vdash \E_C\phi\to\K_C\E_C\phi$ by Lemma~\ref{strategic positive introspection lemma}. Therefore, $\vdash \E_C\phi\to\K_C\S_C\phi$ by the laws of propositional reasoning.
\end{proof}

The next lemma shows that the existence of a know-how strategy by a sub-coalition implies the existence of a know-how strategy by the entire coalition.

\begin{lemma}\label{subset lemma E}
$\vdash\E_C\phi\to \E_D\phi$, where $C\subseteq D$.
\end{lemma}
\begin{proof}
Note that $\phi\to\phi$ is a propositional tautology. Thus, $\vdash\phi\to\phi$. Hence, $\vdash\E_{D\setminus C}(\phi\to\phi)$ by Strategic Necessitation inference rule. At the same time, by Epistemic Cooperation axiom,
$
\vdash\E_{D\setminus C}(\phi\to\phi)\to(\E_C\phi\to\E_D\phi)
$
due to the assumption $C\subseteq D$.  Therefore, $\vdash\E_C\phi\to\E_D\phi$ by Modus Ponens inference rule.
\end{proof}

Although our logical system has three modalities, the system contains necessitation inference rules 
only for two of them. The lemma below shows that the necessitation rule for the third modality is admissible.
\begin{lemma}\label{s necessitation}
For each finite $C\subseteq \mathcal{A}$, inference rule $\dfrac{\phi}{\S_C\phi}$ is admissible in our logical system.
\end{lemma}
\begin{proof}
Assumption $\vdash\phi$ implies $\vdash\E_C\phi$ by Strategic Necessitation inference rule. Hence, $\vdash\S_C\phi$ by Strategic Truth axiom and Modus Ponens inference rule.
\end{proof}

The next result is a counterpart of Lemma~\ref{subset lemma E}. It states that the existence of a strategy by a sub-coalition implies the existence of a strategy by the entire coalition.

\begin{lemma}\label{subset lemma S}
$\vdash\S_C\phi\to \S_D\phi$, where $C\subseteq D$.
\end{lemma}
\begin{proof}
Note that $\phi\to\phi$ is a propositional tautology. Thus, $\vdash\phi\to\phi$. Hence,  $\vdash\S_{D\setminus C}(\phi\to\phi)$ by Lemma~\ref{s necessitation}.
At the same time, by Cooperation axiom,
$
\vdash\S_{D\setminus C}(\phi\to\phi)\to(\S_C\phi\to\S_D\phi)
$
due to the assumption $C\subseteq D$.  Therefore, $\vdash\S_C\phi\to\S_D\phi$ by Modus Ponens inference rule.
\end{proof}

\section{Conclusion}\label{conclusion section}

In this paper we proposed a sound and complete logic system that captures an interplay between the  distributed knowledge, coalition strategies, and how-to strategies. In the future work we hope to explore know-how strategies of non-homogeneous coalitions in which different members contribute differently to the goals of the coalition. For example, ``incognito" members of a coalition might contribute only by sharing information, while ``open" members also contribute by voting.

\bibliographystyle{eptcs}
\bibliography{main}

\end{document}